\documentclass[graybox]{svmult}

\usepackage{mathptmx}       
\usepackage{helvet}         
\usepackage{courier}        
\usepackage{type1cm}        
\usepackage{makeidx}         
\usepackage{graphicx}        
\usepackage{multicol}        
\usepackage[bottom]{footmisc}

\usepackage{amsmath} 
\usepackage{amssymb}
\usepackage{array}
\usepackage{wrapfig}
\usepackage{algorithm}
\usepackage{algorithmic}
\usepackage{hyperref}
\usepackage{epigraph}

\usepackage{bm}
\newcommand{\xSrc}{\vec{x}}
\newcommand{\ySrc}{y}
\newcommand{\xTar}{\vec{u}}
\newcommand{\wSrc}{\vec{w}}
\newcommand\mat[1]{\mathbf{#1}}
\renewcommand\vec[1]{\mathbf{#1}}
\newcommand\eg{\emph{e.g.}}
\newcommand\etal{\textit{et al.}}
\newcommand{\muPosSrc}{\mathbf{\mu_{1}}}
\newcommand{\muNegSrc}{\vec{\mu_{0}}}

\newcommand{\sigSrc}{\mat{C_S}}
\newcommand{\sigTar}{\mat{C_T}}

\makeindex             

\begin{document}

\title*{Correlation Alignment for Unsupervised Domain Adaptation}
\author{Baochen Sun, Jiashi Feng, and Kate Saenko}
\institute{Baochen Sun \at Microsoft, \email{baochens@gmail.com}
\and Jiashi Feng \at National University of Singapore, \email{elefjia@nus.edu.sg}
\and Kate Saenko \at Boston University, \email{saenko@bu.edu}}
%
%
\maketitle

\abstract{In this chapter, we present CORrelation ALignment (CORAL), a simple yet effective method for unsupervised domain adaptation. CORAL minimizes domain shift by aligning the second-order statistics of source and target distributions, without requiring any target labels. In contrast to subspace manifold methods, it aligns the original feature distributions of the source and target domains, rather than the bases of  lower-dimensional subspaces. It is also much simpler than other distribution matching methods. CORAL performs remarkably well in extensive evaluations on standard benchmark datasets. We first describe a solution that applies a linear transformation to source features to align them with target features before classifier training. For linear classifiers, we propose to equivalently apply CORAL to the classifier weights, leading to added efficiency when the number of classifiers is small but the number and dimensionality of target examples are very high. The resulting CORAL Linear Discriminant Analysis (CORAL-LDA) outperforms LDA by a large margin on standard domain adaptation benchmarks. Finally, we extend CORAL to learn a nonlinear transformation that aligns correlations of layer activations in deep neural networks (DNNs). The resulting Deep CORAL approach works seamlessly with DNNs and achieves state-of-the-art performance on standard benchmark datasets. Our code is available at:~\url{https://github.com/VisionLearningGroup/CORAL}}
\section{Introduction}
\label{sec:intro}

Machine learning is very different from human learning. Humans are able to learn from very few labeled examples and apply the learned knowledge to new examples in novel conditions. In contrast, traditional machine learning algorithms assume that the training and test data are independent and identically distributed (i.i.d.) and supervised machine learning methods only perform well when the given extensive labeled data are from the same distribution as the test distribution. However, this assumption rarely holds in practice, as the data are likely to change over time and space. To compensate for the degradation in performance due to domain shift,~\emph{domain adaptation}~\cite{saenko2010adapting,sa,gfk,long_cvpr,bmvc15,coral,tzeng_arxiv15,daume} tries to~\emph{transfer} knowledge from source (training) domain to target (test) domain. Our approach is in line with most existing unsupervised domain adaptation approaches in that we first transform the source data to be as close to the target data as possible. Then a classifier trained on the transformed source domain is applied to the target data.

In this chapter, we mainly focus on the unsupervised scenario as we believe that leveraging the unlabeled target data is key to the success of domain adaptation. In real world applications, unlabeled target data are often much more abundant and easier to obtain. On the other side, labeled examples are very limited and require human annotation. So the question of how to utilize the unlabeled target data is more important for practical visual domain adaptation. For example, it would be more convenient and applicable to have a pedestrian detector automatically adapt to the changing visual appearance of pedestrians rather than having a human annotator label every frame that has a different visual appearance.

Our goal is to propose a simple yet effective approach for domain adaptation that researchers with little knowledge of domain adaptation or machine learning could easily integrate into their application. In this chapter, we describe a ``frustratingly easy'' (to borrow a phrase from~\cite{daume})~\emph{unsupervised} domain adaptation method called CORrelation ALignment or CORAL~\cite{coral} in short. CORAL aligns the input feature distributions of the source and target domains by minimizing the difference between their second-order statistics. The intuition is that we want to capture the structure of the domain using feature correlations. As an example, imagine a target domain of Western movies where most people wear hats; then the ``head'' feature may be positively correlated with the ``hat'' feature. Our goal is to transfer such correlations to the source domain by transforming the source features.

In Section~\ref{sec:coral}, we describe a linear solution to CORAL, where the distributions are aligned by re-coloring whitened source features with the covariance of the target data distribution. This solution is simple yet efficient, as the only computations it needs are (1) computing covariance statistics in each domain and (2) applying the whitening and re-coloring linear transformation to the source features. Then, supervised learning proceeds as usual--training a classifier on the transformed source features. 

For linear classifiers, we can equivalently apply the CORAL transformation to the classifier weights, leading to better efficiency when the number of classifiers is small but the number and dimensionality of the target examples are very high. We present the resulting CORAL--Linear Discriminant Analysis (CORAL-LDA)~\cite{coral-lda} in Section~\ref{sec:coral-lda}, and show that it outperforms standard Linear Discriminant Analysis (LDA) by a large margin on cross domain applications. We also extend CORAL to work seamlessly with deep neural networks by designing a layer that consists of a differentiable CORAL loss~\cite{dcoral}, detailed in Section~\ref{sec:dcoral}. On the contrary to the linear CORAL, Deep CORAL learns a nonlinear transformation and also provides end-to-end adaptation. Section~\ref{sec:exp} describes extensive quantitative experiments on several benchmarks, and Section~\ref{sec:con} concludes the chapter.
\section{Linear Correlation Alignment}
\label{sec:coral}

In this section, we present CORrelation ALignment (CORAL) for unsupervised domain adaptation and derive a linear solution. We first describe the formulation and derivation, followed by the main linear CORAL algorithm and its relationship to existing approaches. In this section and Section~\ref{sec:coral-lda}, we constrain the transformation to be linear. Section~\ref{sec:dcoral} extends CORAL to learn a nonlinear transformation that aligns correlations of layer activations in deep neural networks (Deep CORAL).

\subsection{Formulation and Derivation}
We describe our method by taking a multi-class classification problem as the running example. Suppose we are given source-domain training examples $D_S=\{\xSrc_i\}$, $\xSrc\in\mathbb{R}^d$ with labels $L_S=\{\ySrc_i\}$, $\ySrc \in\{1,...,L\}$, and target data $D_T=\{\xTar_i\}$, $\xTar \in \mathbb{R}^d$. Here both $\xSrc$ and $\xTar$ are the $d$-dimensional feature representations $\phi(I)$ of input $I$. Suppose $\mu_s,\mu_t$ and $C_{S}, C_{T}$ are the feature vector means and covariance matrices. Assuming that all features are normalized to have zero mean and unit variance, $\mu_t=\mu_s=0$ after the normalization step, while $C_{S} \neq C_{T}$.

To minimize the distance between the second-order statistics (covariance) of the source and target features, we apply a linear transformation $A$ to the original source features and use the Frobenius norm as the matrix distance metric:
      \begin{equation}
      \begin{aligned}
      &~\underset{A}{\min} {\| C_{\hat{S}} - C_{T} \|}^2_F\\
      &= \underset{A}{\min} {\| A^{\top}C_{S}A - C_{T} \|}^2_F
      \end{aligned}
      \label{eq:obj}
      \end{equation}
where $C_{\hat{S}}$ is covariance of the transformed source features $D_sA$ and ${\|\cdot\|}^2_F$ denotes the squared matrix Frobenius norm. 

If $\mathrm{rank}(C_S) \geq \mathrm{rank}(C_T)$, then an analytical solution can be obtained by choosing $A$ such that $C_{\hat{S}}= C_{T}$.
However, the data typically lie on a lower dimensional manifold~\cite{outlooks,gfk,sa}, and so the covariance matrices are likely to be low rank~\cite{who}. We derive a solution for this general case, using the following lemma.
\begin{lemma} 
\label{lemma:svt}
Let $Y$ be a real matrix of rank $r_Y$ and X a real matrix of rank at most $r$, where ${r}\leqslant{r_Y}$; let $Y={U_Y}{\Sigma_Y}{V_Y}$ be the SVD of $Y$, and ${\Sigma_{Y[1:r]}}$, $U_{Y[1:r]}$, $V_{Y[1:r]}$ be the largest $r$ singular values and the corresponding left and right singular vectors of $Y$ respectively. Then, $X^{*} = U_{Y[1:r]}{\Sigma_{Y[1:r]}}{V_{Y[1:r]}}^{\top}$ is the optimal solution to the problem of $\underset{X}\min{\| X - Y \|}^2_F$.~\cite{SVT}
\end{lemma}

\begin{theorem} 
Let $\Sigma^{+}$ be the Moore-Penrose pseudoinverse of $\Sigma$, $r_{C_S}$ and $r_{C_T}$ denote the rank of $C_S$ and $C_T$ respectively.
Then, $A^{*} = U_{S}{\Sigma_S^{+}}^{\frac{1}{2}}{U_{S}}^{\top} U_{T[1:r]}{\Sigma_{T[1:r]}}^{\frac{1}{2}}{U_{T[1:r]}}^{\top}$ is the optimal solution to the problem in Equation~\eqref{eq:obj} with $r = \min(r_{C_S}, r_{C_T})$.
\end{theorem}
\begin{proof}  Since $A$ is a linear transformation, $A^{\top}C_SA$ does not increase the rank of $C_S$. Thus, $r_{C_{\hat{S}}}\leqslant{r_{C_{S}}}$. Since $C_S$ and $C_T$ are symmetric matrices, conducting SVD on $C_S$ and $C_T$ gives $C_S=U_S{\Sigma_S}{U_S}^{\top}$ and $C_T = U_T\Sigma_TU_T^\top$ respectively. We first find the optimal value of $C_{\hat{S}}$ through considering the following two  cases:
\vspace{-0.05in}
\begin{case}
$r_{C_S}>{r_{C_T}}$. The optimal solution is $C_{\hat{S}} = C_T$. Thus, $C_{\hat{S}} = U_{T}{\Sigma_{T}}{U_{T}}^{\top} = U_{T[1:r]}{\Sigma_{T[1:r]}}{U_{T[1:r]}}^{\top}$ is the optimal solution to Equation~\eqref{eq:obj} where $r = r_{C_T}$.
\end{case}
\vspace{-0.05in}
\begin{case}
$r_{C_S}\leqslant{r_{C_T}}$. Then, according to Lemma \ref{lemma:svt}, $C_{\hat{S}} = U_{T[1:r]}{\Sigma_{T[1:r]}}{U_{T[1:r]}}^{\top}$ is the optimal solution to Equation~\eqref{eq:obj} where $r = r_{C_S}$.
\end{case}
Combining the results in the above two cases yields that ${{C_{\hat{S}}}} = U_{T[1:r]}{\Sigma_{T[1:r]}}{U_{T[1:r]}}^{\top}$ is the optimal solution to Equation~\eqref{eq:obj} with $r = \min(r_{C_S}, r_{C_T})$.
We then proceed to solve for $A$ based on the above result. 
Letting $C_{\hat{S}} = {A}^{\top}C_S A$, we can write
\begin{equation*}
{A}^{\top}C_{S}{A}= U_{T[1:r]}{\Sigma_{T[1:r]}}{U_{T[1:r]}}^{\top}.
\end{equation*}
Since $C_S = U_S \Sigma_S {U_S}^\top$, we have
\begin{equation*}
{A}^{\top}U_S{\Sigma_S}{U_S}^{\top}{A}= U_{T[1:r]}{\Sigma_{T[1:r]}}{U_{T[1:r]}}^{\top}.
\end{equation*}
This gives:
\begin{equation*}
{({U_S}^{\top}A)}^{\top}{\Sigma_S}({U_S}^{\top}{A})= U_{T[1:r]}{\Sigma_{T[1:r]}}{U_{T[1:r]}}^{\top}.
\end{equation*}
Let $E ={\Sigma_S^{+}}^{\frac{1}{2}} {U_S}^{\top} {U_{T[1:r]}} {\Sigma_{T[1:r]}}^{\frac{1}{2}}{U_{T[1:r]}}^{\top}$, then the right hand side of the above equation can be re-written as ${E}^{\top}{\Sigma_S}E$. This gives
\begin{align*}
&{({U_S}^{\top}A)}^{\top}{\Sigma_S}({U_S}^{\top}{A})= {E}^{\top}{\Sigma_S}E
\end{align*}
By setting ${U_S}^{\top}A$ to $E$, we obtain the optimal solution of $A$ as 
\begin{equation} 
\begin{aligned}
A^{*}&={U_S}E\\
&=(U_{S}{\Sigma_S^{+}}^{\frac{1}{2}}{U_{S}}^{\top})(U_{T[1:r]}{\Sigma_{T[1:r]}}^{\frac{1}{2}}{U_{T[1:r]}}^{\top}).
 \end{aligned}
\label{eq:slu}
\end{equation}
\end{proof}

\begin{figure}[t]
\centering
\includegraphics[width=0.7\linewidth]{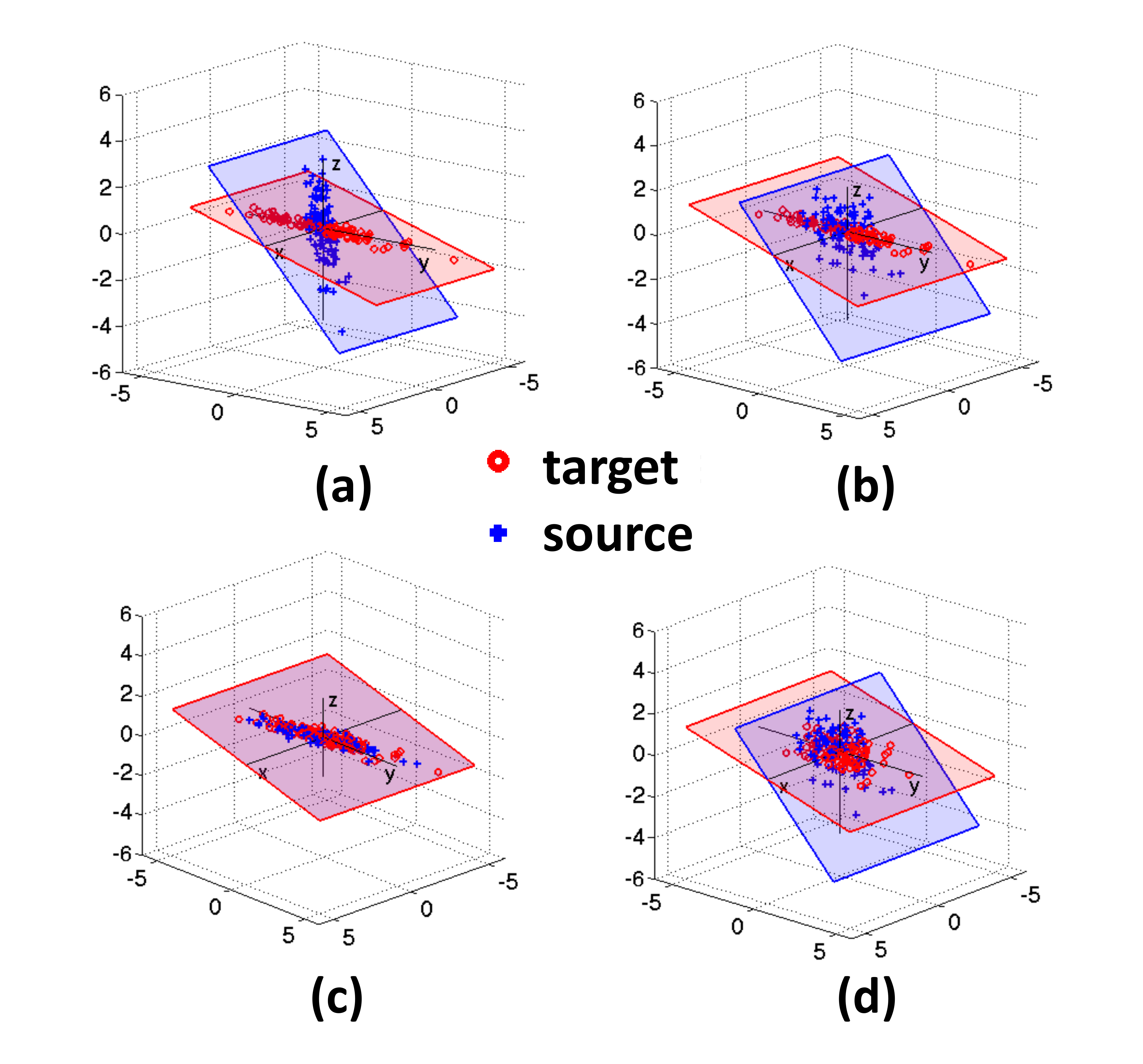}
\caption{\small \textbf{(a-c)} Illustration of CORrelation ALignment (CORAL) for Domain Adaptation: (a) The original source and target domains have different distribution covariances, despite the features being normalized to zero mean and unit standard deviation. This presents a problem for transferring classifiers trained on source to target. (b) The same two domains after source decorrelation, i.e. removing the feature correlations of the source domain. (c) Target re-correlation, adding the correlation of the target domain to the source features. After this step, the source and target distributions are well aligned and the classifier trained on the adjusted source domain is expected to work well in the target domain. \textbf{(d)} One might instead attempt to align the distributions by whitening both source and target. However, this will fail since the source and target data are likely to lie on different subspaces due to domain shift. (Best viewed in color)}
\label{fig:variance}
\end{figure}

\subsection{Algorithm}
\label{subsec:algo}
Figures~\ref{fig:variance}(a-c) illustrate the linear CORAL approach. Figure~\ref{fig:variance}(a) shows example original source and target data distributions. We can think of the transformation $A$ intuitively as follows: the first part $U_{S}{\Sigma_S^{+}}^{\frac{1}{2}}{U_{S}}^{\top}$ whitens the source data, while the second part $U_{T[1:r]}{\Sigma_{T[1:r]}}^{\frac{1}{2}}{U_{T[1:r]}}^{\top}$ re-colors it with the target covariance. This is illustrated in Figure~\ref{fig:variance}(b) and Figure~\ref{fig:variance}(c) respectively.

In practice, for the sake of efficiency and stability, we can avoid the expensive SVD steps and perform traditional data whitening and coloring.
Traditional whitening adds a small regularization parameter $\lambda$ to the diagonal elements of the covariance matrix to explicitly make it full rank and then multiplies the original features by its inverse square root (or square root for coloring.) 
This is advantageous because: (1) 
it is faster\footnote{the entire CORAL transformation takes less than one minute on a regular laptop for dimensions as large as $D_S\in\mathbb{R}^{795\times4096}$ and $D_T\in\mathbb{R}^{2817\times4096}$} and more stable, as SVD on the original covariance matrices might not be stable and might be slow to converge; (2) as illustrated in Figure~\ref{fig:sens}, the performance is similar to the analytical solution in Equation~\eqref{eq:slu} and very stable with respect to~$\lambda$. In the experiments provided at the end of this chapter we set~$\lambda$ to 1. The final algorithm can be written in four lines of MATLAB code as illustrated in Algorithm~\ref{alg:coral}.
 
\begin{figure}
\centering
\includegraphics[width=0.5\columnwidth]{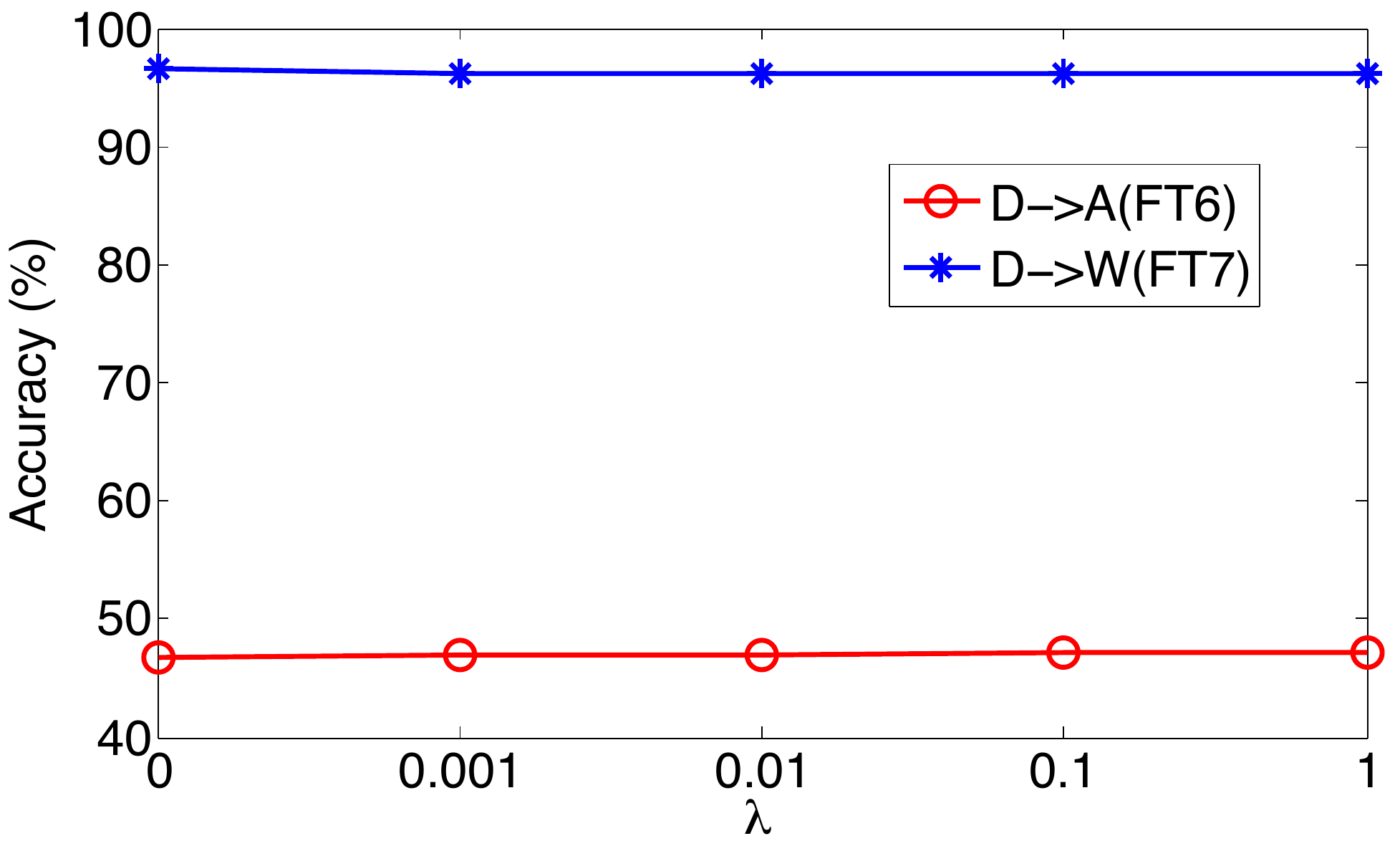}
\caption{\small Sensitivity of CORAL to the covariance regularization parameter~$\lambda$ with~$\lambda \in$ \{0, 0.001, 0.01, 0.1, 1\}. The plots show classification accuracy on target data for two domain shifts (blue and red). When $\lambda = 0$, there is no regularization and we use the analytical solution in Equation~\eqref{eq:slu}. Please refer to Section~\ref{subsec:recog} for details of the experiment.}
\label{fig:sens}
\end{figure}

\begin{algorithm}
\caption{CORAL for Unsupervised Domain Adaptation}
\begin{small}
\begin{algorithmic} 
\STATE \textbf{Input:} Source Data $D_S$, Target Data $D_T$
\STATE \textbf{Output:} Adjusted Source Data $D_{s}^{*}$
\STATE $C_S = cov(D_S) + eye(size(D_S, 2))$
\STATE $C_T = cov(D_T) + eye(size(D_T, 2))$
\STATE $D_S = D_S*C_S^{\frac{-1}{2}}$  ~~~~~~~~~~~~~~~~~~~~~~~~~~~~~~~~\% whitening source
\STATE $D_{S}^{*} = D_S*C_T^{\frac{1}{2}}$  ~~~~~~~~~~~~~~~~~~~~~~~~~~~~~~~~~\% re-coloring with target covariance
\end{algorithmic} 
\end{small}
\label{alg:coral}
\end{algorithm}

One might instead attempt to align the distributions by whitening both source and target. As shown in Figure~\ref{fig:variance}(d), this will fail as the source and target data are likely to lie on different subspaces due to domain shift. An alternative approach would be whitening the target and then re-coloring it with the source covariance. However, as demonstrated in~\cite{outlooks,sa} and our experiments, transforming data from source to target space gives better performance. 
This might be due to the fact that by transforming the source to target space the classifier is trained using both the label information from the source and the unlabelled structure from the target.

After CORAL transforms the source features to the target space, a classifier $f_{\wSrc}$ parametrized by $\wSrc$ can be trained on the adjusted source features and directly applied to target features. For a linear classifier $f_{\wSrc}(I) =  \wSrc^T \phi(I)$, we can apply an equivalent transformation to the parameter vector $\wSrc$ (\eg, $f_{\wSrc}(I) =  (\wSrc^TA) \phi(I)$) instead of the features (\eg, $f_{\wSrc}(I) =  \wSrc^T(A\phi(I))$). This results in added efficiency when the number of classifiers is small but the number and dimensionality of target examples is very high. For linear SVM, this extension is straightforward. In Section~\ref{sec:coral-lda}, we apply the idea of CORAL to another commonly used linear classifier--Linear Discriminant Analysis (LDA). LDA is special in the sense that its weights also use the covariance of the data. It is also extremely efficient for training a large number of classifiers~\cite{who}. 

\subsection{Relationship to Existing Methods} 
It has long been known that input feature normalization improves many machine learning methods, \eg,~\cite{batchnorm}. However, CORAL does not simply perform feature normalization, but rather aligns two different distributions. Batch Normalization~\cite{batchnorm} tries to compensate for~\emph{internal} covariate shift by normalizing each mini-batch to be zero-mean and unit-variance. However, as illustrated in Figure~\ref{fig:variance}(a), such normalization might not be enough. Even if used with full whitening, Batch Normalization may not compensate for~\emph{external} covariate shift: the layer activations will be decorrelated for a source point but not for a target point. What's more, as mentioned in Section~\ref{subsec:algo}, whitening both domains is not a successful strategy.

Recent state-of-the-art unsupervised approaches project the source and target distributions into a lower-dimensional manifold and find a transformation that brings the subspaces closer together~\cite{gopalan-iccv11,gfk,sa,outlooks}. CORAL avoids subspace projection, which can be costly and requires selecting the hyper-parameter that controls the dimensionality of the subspace, $k$. We note that subspace-mapping approaches~\cite{outlooks,sa} only align the top $k$ eigenvectors of the source and target covariance matrices. On the contrary, CORAL aligns the covariance matrices, which can only be re-constructed using all eigenvectors and eigenvalues. Even though the eigenvectors can be aligned well, the distributions can still differ a lot due to the difference of eigenvalues between the corresponding eigenvectors of the source and target data. CORAL is a more general and much simpler method than the above two as it takes into account~\emph{both} eigenvectors~\emph{and} eigenvalues of the covariance matrix without the burden of subspace dimensionality selection. 

Maximum Mean Discrepancy (MMD) based methods (\eg, TCA~\cite{tca}, DAN~\cite{dan_long15}) for domain adaptation can be interpreted as ``moment matching'' and can express arbitrary statistics of the data. Minimizing MMD with a polynomial kernel ($k(x,y) = (1+x'y)^q$ with $q=2$) is similar to the CORAL objective, however, no previous work has used this kernel for domain adaptation nor proposed a closed form solution to the best of our knowledge. 
The other difference is that MMD based approaches usually apply the~\emph{same} transformation to both the source and target domain. As demonstrated in~\cite{ref:kulis_cvpr11,outlooks,sa}, asymmetric transformations are more flexible and often yield better performance for domain adaptation tasks. Intuitively, symmetric transformations find a space that ``ignores'' the differences between the source and target domain while asymmetric transformations try to ``bridge'' the two domains.

\section{CORAL Linear Discriminant Analysis}
\label{sec:coral-lda}

In this section, we introduce how CORAL can be applied for aligning multiple linear classifiers. In particular, we take LDA as the example for illustration, considering LDA is  a commonly used and effective linear classifier. Combining CORAL and LDA gives a new efficient adpative learning approach CORAL-LDA.  We use the task of object detection as a running example to explain CORAL-LDA.

We begin by describing the decorrelation-based approach to detection proposed in~\cite{who}. Given an image $I$, it follows the sliding-window paradigm, extracting a $d$-dimensional feature vector $\phi(I,b)$ at each window $b$ across all locations and at multiple scales. 
It then scores the windows using a scoring function 
      \begin{align}
      \label{chap4_eq:score}
         f_{\wSrc}(I,b) &=  \wSrc^{\top}  \phi(I,b). 
      \end{align}
In practice, all windows with values of $f_{\vec{w}}$ above a predetermined threshold are considered positive detections.

In recent years, use of the linear SVM as the scoring function $f_{\wSrc}$, usually with Histogram of Gradients (HOG) as the features $\phi$, has emerged as the predominant object detection paradigm. Yet, as observed by Hariharan~\etal~\cite{who}, training SVMs can be expensive, especially because it usually involves costly rounds of hard negative mining. Furthermore, the training must be repeated for each object category, which makes it scale poorly with the number of categories.

Hariharan~\etal~\cite{who} proposed a much more efficient alternative, learning $f_{\wSrc}$ with Linear Discriminant Analysis (LDA). LDA is a well-known linear classifier that models the training set of examples $\xSrc$ with labels $\ySrc \in\{0,1\}$ as being generated by $p(\xSrc,\ySrc)=p(\xSrc|\ySrc)p(\ySrc)$. $p(\ySrc)$ is the prior on class labels and the class-conditional densities are normal distributions 
      \begin{align}
      \label{eq:lda1}
         p(\xSrc|\ySrc) &= \mathcal{N}(\xSrc; {\vec{\mu}}^{\ySrc}, \sigSrc),
      \end{align}
where the feature vector covariance $\sigSrc$ is assumed to be the same for both positive and negative (background) classes. In our case, the feature is represented by $\xSrc = \phi(I,b)$. The resulting classifier is given by
\vspace{-0.1in}
      \begin{align}
      \label{eq:lda2}
         \wSrc &= {\sigSrc}^{-1} ( \muPosSrc - \muNegSrc)
      \end{align}
The innovation in \cite{who} was to re-use $\sigSrc$ and $\muNegSrc$, the background mean, for all categories, reducing the task of learning a new category model to computing the average positive feature, $\muPosSrc$. This was accomplished by calculating $\sigSrc$ and $\muNegSrc$ for the largest possible window and subsampling to estimate all other smaller window sizes. Also, $\sigSrc$ was shown to have a sparse local structure, with correlation falling off sharply beyond a few nearby image locations.

Like other classifiers, LDA learns to suppress non-discriminative structures and enhance the contours of the object. However it does so by learning the global covariance statistics once for all natural images, and then using the inverse covariance matrix to remove the non-discriminative correlations, and the negative mean to remove the average feature. LDA was shown in~\cite{who} to have competitive performance to SVM, and can be implemented both as an exemplar-based~\cite{exemplarsvm} or as deformable parts model (DPM)~\cite{dpm}.

\begin{figure}[t]
\centering
\includegraphics[width=0.8\linewidth]{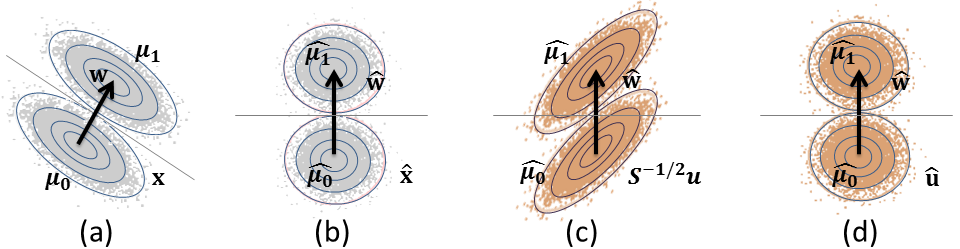}
\caption{\small (a) Applying a linear classifier $\wSrc$ learned by LDA to source data $\xSrc$  is equivalent to (b) applying classifier $\hat{\wSrc}=\sigSrc^{-1/2}\wSrc$ to decorrelated points $\sigSrc^{-1/2}\xSrc$. (c) However, target points $\xTar$ may still be correlated after $\sigSrc^{-1/2}\xTar$, hurting performance. (d) Our method uses target-specific covariance to obtain properly decorrelated $\hat{\xTar}$. }
\label{chap4_fig:adapt}
\end{figure}

We observe that estimating global statistics  $\sigSrc$ and $\muNegSrc$ once and re-using them for all tasks may work when training and testing in the same domain, but in our case, the source training data is likely to have different statistics from the target data. Figure~\ref{chap4_fig:wrongdomain} illustrates the effect of centering and decorrelating a positive mean using global statistics from the wrong domain. The effect is clear: important discriminative information is removed while irrelevant structures are not. 

Based on this observation, we propose an adaptive decorrelation approach to detection. Assume that we are given labeled training data $\{ \xSrc, \ySrc \}$ in the source domain (\eg, virtual images rendered from 3D models), and unlabeled examples $\xTar$ in the target domain (\eg, real images collected in an office environment). Evaluating the scoring function $f_{\wSrc}(\xSrc)$ in the source domain is equivalent to first de-correlating the training features 
 $\hat{\xSrc} = {\sigSrc}^{-1/2} \xSrc$, 
computing their positive and negative class means 
$\hat{\muPosSrc} = {\sigSrc}^{-1/2} \muPosSrc $ 
and 
$\hat{\muNegSrc} = {\sigSrc}^{-1/2} \muNegSrc $ 
and then projecting the decorrelated feature onto the decorrelated difference between means, $f_{\wSrc}(\xSrc) = \hat{\vec{w}}^{\top} \hat{\xSrc}$, where $\hat{\wSrc}=(\hat{\muPosSrc} - \hat{\muNegSrc})$.
This is illustrated in Figure~\ref{chap4_fig:adapt}(a-b).

However, as we saw in Figure~\ref{chap4_fig:wrongdomain}, the assumption that the input is properly decorrelated does not hold if the input comes from a target domain with a different covariance structure. Figure~\ref{chap4_fig:adapt}(c) illustrates this case, showing that ${\sigSrc}^{-1/2} \xTar$ does not have isotropic covariance. Therefore, $\wSrc$ cannot be used directly. 

\begin{figure}
\centering
\includegraphics[width=0.5\linewidth]{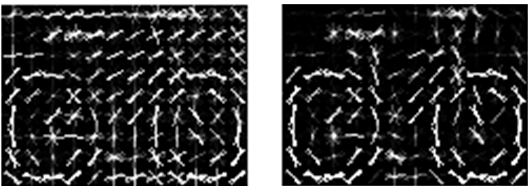}
\caption{\small Visualization of classifier weights of bicycle (decorrelated with mismatched-domain covariance (left) v.s. with same-domain covariance (right)).}
\label{chap4_fig:wrongdomain}
\end{figure}

We may be able to compute the covariance of the target domain on the unlabeled target points $\xTar$, but not the positive class mean. Therefore, we would like to re-use the decorrelated mean difference $\hat{\wSrc}$, but adapt to the covariance of the target domain.
In the rest of the chapter, we make the assumption that the difference between positive and negative means is the same in the source and target. This may or may not hold in practice, and we discuss this further in Section~\ref{sec:exp}.

Let the estimated target covariance be $\sigTar$. We first decorrelate the target input feature with its inverse square root, and then apply $\hat{\wSrc}$ directly, as shown in Figure~\ref{chap4_fig:adapt}(d). The resulting scoring function is:

      \begin{equation}
      \begin{aligned}
         f_{\hat{\wSrc}}(\xTar) &= \hat{\wSrc}^{\top} \hat{\xTar}  \\
			  &= {({\sigSrc}^{-1/2} (\muPosSrc-\muNegSrc))}^{\top} ({\sigTar}^{-1/2} \xTar) \\
			 &= {(({\sigTar}^{-1/2})^{\top}{\sigSrc}^{-1/2} (\muPosSrc-\muNegSrc))}^{\top}\xTar
	  \end{aligned}
	  \label{eq:lda4}
      \end{equation}
      
This corresponds to a transformation 
$({\sigTar}^{-1/2})^{\top} ({\sigSrc}^{-1/2})$
instead of the original whitening ${\sigSrc}^{-1}$  being applied to the difference between means to compute $\wSrc$. Note that if source and target domains are the same, then 
$({\sigTar}^{-1/2})^{\top} ({\sigSrc}^{-1/2})$ 
equals to ${\sigSrc}^{-1}$ since both ${\sigTar}$ and $\sigSrc$ are symmetric. In this case, Equation~\ref{eq:lda4} ends up the same as Equation~\ref{eq:lda2}.

In practice, either the source or the target component of the above transformation may also work, or even statistics from similar domains. However, as we will see in Section~\ref{subsec:coral-lda}, dissimilar domain statistics can significantly hurt performance. Furthermore, if either source or target has only images of the positive category available, and cannot be used to properly compute background statistics, the other domain can still be used.

CORAL-LDA works in a purely unsupervised way. Here, we extend it to semi-supervised adaptation when a few labeled examples are available in the target domain. Following~\cite{ICRA14}, a simple adaptation method is used whereby the template learned on source positives is combined with a template learned on target positives, using a weighted combination. The key difference with our approach is that the target template uses target-specific statistics.

In~\cite{ICRA14}, the author uses the same background statistics as~\cite{who} which were estimated on 10,000 natural images from the PASCAL VOC 2010 dataset. Based on our analysis above, even though these background statistics were estimated from a very large amount of real image data, it will not work for all domains. In section~\ref{subsec:coral-lda}, our results confirm this claim.
\section{Deep CORAL}
\label{sec:dcoral}
In this section, we extend CORAL to work seamlessly with deep neural networks by designing a differentiable CORAL loss. Deep CORAL enables end-to-end adaptation and also learns more a powerful nonlinear transformation. It can be easily integrated into different layers or network architectures. Figure~\ref{fig:d-coral} shows a sample Deep CORAL architecture using our proposed correlation alignment layer for deep domain adaptation. We refer to Deep CORAL as any deep network incorporating the CORAL loss for domain adaptation. 

\begin{figure}[t]
\centering
\includegraphics[width=0.8\linewidth]{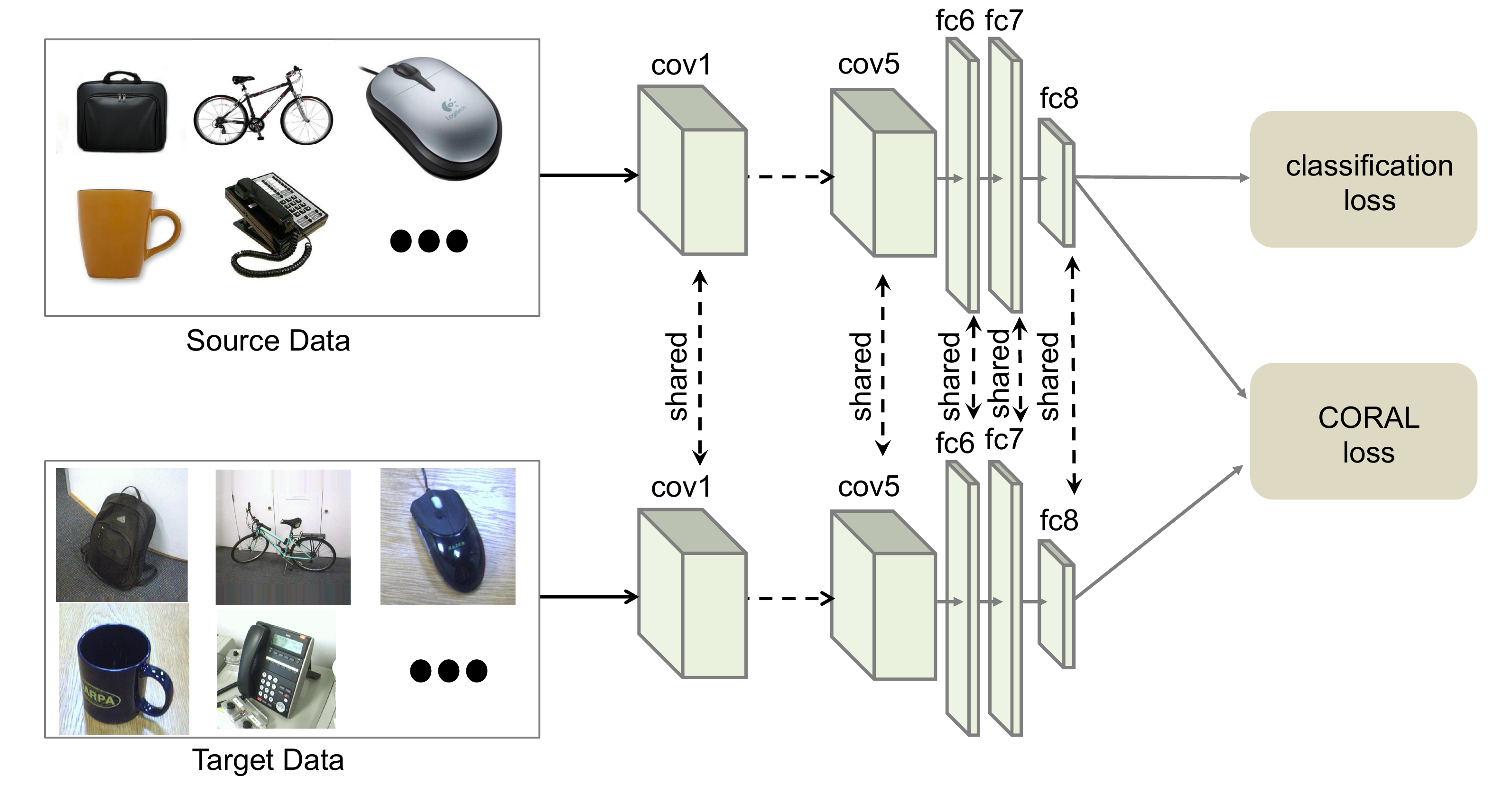}
\caption{\small Sample Deep CORAL architecture based on a CNN with a classifier layer. For generalization and simplicity, here we apply the CORAL loss to the $fc8$ layer of AlexNet~\cite{alexnet}. Integrating it into other layers or network architectures is also possible.}
\label{fig:d-coral}
\end{figure}

We first describe the CORAL loss between two domains for a single feature layer. Suppose the numbers of source and target data are $n_{S}$ and $n_{T}$ respectively. Here both $\xSrc$ and $\xTar$ are the $d$-dimensional deep layer activations $\phi(I)$ of input $I$ that we are trying to learn. Suppose $D_S^{ij}$ ($D_T^{ij}$) indicates the $j$-th dimension of the $i$-th source (target) data example and $C_{S}$ ($C_{T}$) denote the feature covariance matrices. 

We define the CORAL loss as the distance between the second-order statistics (covariances) of the source and target features:
      \begin{equation}
      \begin{aligned}
      {\mathcal{L}_{CORAL}}= {\frac{1}{4d^2}}{\| C_{S} - C_{T} \|}^2_F\\
      \end{aligned}
      \label{eq:coral}
      \end{equation}
where ${\|\cdot\|}^2_F$ denotes the squared matrix Frobenius norm. 
The covariance matrices of the source and target data are given by:
      \begin{equation}
      \begin{aligned}
      &C_{S}= {\frac{1}{n_{S}-1}}({D_S^{\top} D_S - \frac{1}{n_{S}}{({\textbf{1}}^{\top}D_S})^{\top}{({\textbf{1}}^{\top}D_S})})
      \end{aligned}
      \label{eq:cov_s}
      \end{equation}
      \begin{equation}
      \begin{aligned}
      &C_{T}= {\frac{1}{n_{T}-1}}({D_T^{\top} D_T - \frac{1}{n_{T}}{({\textbf{1}}^{\top}D_T})^{\top}{({\textbf{1}}^{\top}D_T})})
      \end{aligned}
      \label{eq:cov_t}
      \end{equation}
where $\textbf{1}$ is a column vector with all elements equal to 1. 

The gradient with respect to the input features can be calculated using the chain rule:
      \begin{equation}
      \begin{aligned}
      &\frac{\partial{\mathcal{L}_{CORAL}}}{\partial{D_S^{ij}}}=\frac{1}{d^2(n_S-1)}((D_S^{\top}-\frac{1}{n_{S}}({{\textbf{1}}^{\top}D_S})^{\top}{\textbf{1}}^{\top})^{\top}(C_{S} - C_{T}))^{ij}
      \end{aligned}
      \label{eq:gradient_s}
      \end{equation}
      \begin{equation}
      \begin{aligned}
      \frac{\partial{\mathcal{L}_{CORAL}}}{\partial{D_T^{ij}}}=-\frac{1}{d^2(n_T-1)}((D_T^{\top}-\frac{1}{n_{T}}({{\textbf{1}}^{\top}D_T})^{\top}{\textbf{1}}^{\top})^{\top}(C_{S} - C_{T}))^{ij}
      \end{aligned}
      \label{eq:gradient_t}
      \end{equation}
In our experiments, we use batch covariances and the network parameters are shared between the two networks, but other settings are also possible.

To see how this loss can be used to adapt an existing neural network, let us return to the multi-class classification problem. Suppose we start with a network with a final classification layer, such as the ConvNet shown in Figure~\ref{fig:d-coral}.
As mentioned before, the final deep features need to be both discriminative enough to train a strong classifier and invariant to the difference between source and target domains. Minimizing the classification loss itself is likely to lead to overfitting to the source domain, causing reduced performance on the target domain. On the other hand, minimizing the CORAL loss alone might lead to degenerated features. For example, the network could project all of the source and target data to a single point, making the CORAL loss trivially zero. However, no strong classifier can be constructed on these features. Joint training with both the classification loss and CORAL loss is likely to learn features that work well on the target domain:
      \begin{equation}
      \begin{aligned}
      {\mathcal{L}}= {\mathcal{L}_{CLASS}} + \sum_{i=1}^{t}\lambda_{i}{\mathcal{L}_{CORAL_i}}\\
      \end{aligned}
      \label{eq:obj}
      \end{equation}
where $t$ denotes the number of CORAL loss layers in a deep network and $\lambda_i$ is a weight that trades off the adaptation with classification accuracy on the source domain. As we show below, these two losses play counterparts and reach an~\emph{equilibrium} at the end of training, where the final features are discriminative and generalize well to the target domain.
\section{Experiments}
\label{sec:exp}


We evaluate CORAL and Deep CORAL on object recognition~\cite{saenko2010adapting} using standard benchmarks and protocols. In all experiments we assume the target domain is unlabeled. For CORAL, we follow the standard procedure~\cite{sa,decaf} and use a linear SVM as the base classifier. The model selection approach of~\cite{sa} is used to set the $C$ parameter for the SVM by doing cross-validation on the source domain. For CORAL-LDA, as efficiency is the main concern, we evaluate it on the more time constrained task--object detection. We follow the protocol of~\cite{ICRA14} and use HOG features. To have a fair comparison, we use accuracies reported by other authors with exactly the same setting or conduct experiments using the source code provided by the authors.

\subsection{Object Recognition}
\label{subsec:recog}
In this set of experiments, domain adaptation is used to improve the accuracy of an object classifier on novel image domains. Both the standard Office~\cite{saenko2010adapting} and extended Office-Caltech10~\cite{gfk} datasets are used as benchmarks in this chapter. Office-Caltech10 contains 10 object categories from an office environment (\eg, keyboard, laptop, etc.) in 4 image domains: $Webcam$, $DSLR$, $Amazon$, and $Caltech256$. The standard Office dataset contains 31 (the same 10 categories from Office-Caltech10 plus 21 additional ones) object categories in 3 domains: $Webcam$, $DSLR$, and $Amazon$.

\begin{table}
\centering
\resizebox{\columnwidth}{!}{
\begin{tabular}{|l||c|c|c|c|c|c|c|c|c|c|c|c|c|}
\hline
~ & A$\rightarrow$C & A$\rightarrow$D & A$\rightarrow$W & C$\rightarrow$A & C$\rightarrow$D & C$\rightarrow$W & D$\rightarrow$A & D$\rightarrow$C & D$\rightarrow$W & W$\rightarrow$A & W$\rightarrow$C & W$\rightarrow$D & AVG\\ 
\hline
NA & 35.8 & 33.1 & 24.9  &43.7 & 39.4 & 30.0 & 26.4 & 27.1  & 56.4 & 32.3 & 25.7 & 78.9 & 37.8\\ 
\hline
SVMA & 34.8 & 34.1  & 32.5 & 39.1 &  34.5 & 32.9 & 33.4  & 31.4  & 74.4 & 36.6 & 33.5 & 75.0 & 41.0\\ 
\hline
DAM & 34.9 & 34.3  & 32.5 & 39.2 &  34.7 & 33.1 & 33.5  & 31.5  & 74.7 & 34.7 & 31.2 & 68.3 & 40.2\\ 
\hline
GFK & 38.3 & 37.9  & 39.8&44.8 &  36.1 & 34.9 & 37.9  & 31.4  & 79.1 & 37.1 & 29.1 & 74.6 & 43.4\\ 
\hline
TCA & 40.0 & \textbf{39.1}  & \textbf{40.1}  & 46.7 &  \textbf{41.4} & 36.2 & 39.6  & 34.0  & 80.4 & \textbf{40.2} & 33.7 & 77.5 & 45.7\\ 
\hline
SA & 39.9 & 38.8  & 39.6  & 46.1 & 39.4 & 38.9 & \textbf{42.0}  & \textbf{35.0}  & 82.3 & 39.3 & 31.8 & 77.9 & 45.9\\ 
\hline
CORAL & \textbf{40.3} & 38.3 & 38.7 & \textbf{47.2} & 40.7 &\textbf{39.2} & 38.1 & 34.2 & \textbf{85.9} & 37.8 &\textbf{ 34.6} &\textbf{84.9} & \textbf{46.7}\\ 
\hline
\end{tabular}
}
\caption{\small Object recognition accuracies of all 12 domain shifts on the Office-Caltech10 dataset~\cite{gfk} with SURF features, following the protocol of~\cite{gfk,sa,gopalan-iccv11,ref:kulis_cvpr11,saenko2010adapting}.}
\label{tab:result_office-caltech10_surf}
\vspace{-0.1in}
\end{table}

\paragraph{\textbf{Object Recognition with Shallow Features}}
We follow the standard protocol of~\cite{gfk,sa,gopalan-iccv11,ref:kulis_cvpr11,saenko2010adapting} and conduct experiments on the Office-Caltech10 dataset with shallow features (SURF).
The SURF features were encoded with 800-bin bag-of-words histograms and normalized to have zero mean and unit standard deviation in each dimension.  Since there are four domains, there are 12 experiment settings, namely, A$\rightarrow$C (train classifier on (A)mazon, test on (C)altech), A$\rightarrow$D (train on (A)mazon, test on (D)SLR), A$\rightarrow$W, and so on. We follow the standard protocol and conduct experiments in 20 randomized trials for each domain shift and average the accuracy over the trials. In each trial, we use the standard setting~\cite{gfk,sa,gopalan-iccv11,ref:kulis_cvpr11,saenko2010adapting} and randomly sample the same number (20 for $Amazon$, $Caltech$, and $Webcam$; 8 for $DSLR$ as there are only 8 images per category in the $DSLR$ domain) of labelled images in the source domain as training set, and use all the unlabelled data in the target domain as the test set. 

In Table~\ref{tab:result_office-caltech10_surf}, we compare our method to five recent published methods: SVMA~\cite{svma}, DAM~\cite{ref:duan09}, GFK~\cite{gfk}, SA~\cite{sa}, and TCA~\cite{tca} as well as the no adaptation baseline (NA). GFK, SA, and TCA are manifold based methods that project the source and target distributions into a lower-dimensional manifold. GFK integrates over an infinite number of subspaces along the subspace manifold using the kernel trick. SA aligns the source and target subspaces by computing a linear map that minimizes the Frobenius norm of their difference. TCA performs domain adaptation via a new parametric kernel using feature extraction methods by projecting data onto the learned transfer components. DAM introduces smoothness assumption to enforce the target classifier share similar decision values with the source classifiers. Even though these methods are far more complicated than ours and require tuning of hyperparameters (\eg, subspace dimensionality), our method achieves the best average performance across all the 12 domain shifts. Our method also improves on the no adaptation baseline (NA), in some cases increasing accuracy significantly (from 56\% to 86\% for D$\rightarrow$W).

\paragraph{\textbf{Object Recognition with Deep Features}} 
For visual domain adaptation with deep features, we follow the standard protocol of~\cite{gfk,dan_long15,decaf,tzeng_arxiv15,reversegrad} and use all the labeled source data and all the target data without labels on the standard Office dataset~\cite{saenko2010adapting}. Since there are 3 domains, we conduct experiments on all 6 shifts (5 runs per shift), taking one domain as the source and another as the target. 

In this experiment, we apply the CORAL loss to the last classification layer
as it is the most general case--most deep classifier architectures (\eg, convolutional neural networks, recurrent neural networks) contain a fully connected layer for classification. Applying the CORAL loss to other layers or other network architectures is also possible. The dimension of the last fully connected layer ($fc8$) was set to the number of categories (31) and initialized with $\mathcal{N}(0,0.005)$. The learning rate of $fc8$ was set to 10 times the other layers as it was training from scratch. We initialized the other layers with the parameters pre-trained on ImageNet~\cite{imagenet} and kept the original layer-wise parameter settings. In the training phase, we set the batch size to 128, base learning rate to $10^{-3}$, weight decay to $5\times10^{-4}$, and momentum to 0.9. The weight of the CORAL loss ($\lambda$) is set in such way that at the end of training the classification loss and CORAL loss are roughly the same. It seems be a reasonable choice as we want to have a feature representation that is both discriminative and also minimizes the distance between the source and target domains. We used Caffe~\cite{caffe} and BVLC Reference CaffeNet for all of our experiments.

We compare to 7 recently published methods: CNN~\cite{alexnet} (no adaptation), GFK~\cite{gfk}, SA~\cite{sa}, TCA~\cite{tca}, CORAL~\cite{coral}, DDC~\cite{tzeng_arxiv15}, DAN~\cite{dan_long15}. GFK, SA, and TCA are manifold based methods that project the source and target distributions into a lower-dimensional manifold and are not end-to-end deep methods. DDC adds a domain confusion loss to AlexNet~\cite{alexnet} and fine-tunes it on both the source and target domain. DAN is similar to DDC but utilizes a multi-kernel selection method for better mean embedding matching and adapts in multiple layers. For direct comparison, DAN in this paper uses the hidden layer $fc8$. For GFK, SA, TCA, and CORAL, we use the $fc7$ feature fine-tuned on the source domain ($FT7$ in~\cite{coral}) as it achieves better performance than generic pre-trained features, and train a linear SVM~\cite{sa,coral}. To have a fair comparison, we use accuracies reported by other authors with exactly the same setting or conduct experiments using the source code provided by the authors.

From Table~\ref{tab:result_office31_deep} we can see that Deep CORAL (D-CORAL) achieves better average performance than CORAL and the other 6 baseline methods. In 3 out of 6 shifts, it achieves the highest accuracy. For the other 3 shifts, the margin between D-CORAL and the best baseline method is very small ($\leqslant0.7$).

\begin{table}[tb]
\begin{center}
\resizebox{0.8\columnwidth}{!}{
\begin{tabular}{|l||c|c|c|c|c|c|c|}
\hline
~ & A$\rightarrow$D&	A$\rightarrow$W	&D$\rightarrow$A&	D$\rightarrow$W	&W$\rightarrow$A	&W$\rightarrow$D&AVG\\ 
\hline
GFK	&52.4\scriptsize{$\pm$0.0}	&54.7\scriptsize{$\pm$0.0}	&43.2\scriptsize{$\pm$0.0}	&92.1\scriptsize{$\pm$0.0}	&41.8\scriptsize{$\pm$0.0}	&96.2\scriptsize{$\pm$0.0}  &63.4\\
\hline
SA	&50.6\scriptsize{$\pm$0.0}	&47.4\scriptsize{$\pm$0.0}	&39.5\scriptsize{$\pm$0.0}	&89.1\scriptsize{$\pm$0.0}	&37.6\scriptsize{$\pm$0.0}	&93.8\scriptsize{$\pm$0.0}  &59.7\\
\hline
TCA	&46.8\scriptsize{$\pm$0.0}	&45.5\scriptsize{$\pm$0.0}	&36.4\scriptsize{$\pm$0.0}	&81.1\scriptsize{$\pm$0.0}	&39.5\scriptsize{$\pm$0.0}	&92.2\scriptsize{$\pm$0.0}  &56.9\\
\hline
CORAL	&65.7\scriptsize{$\pm$0.0}	&64.3\scriptsize{$\pm$0.0}	&48.5\scriptsize{$\pm$0.0}	&\textbf{96.1}\scriptsize{$\pm$0.0}	&48.2\scriptsize{$\pm$0.0}	&\textbf{99.8}\scriptsize{$\pm$0.0}  &70.4\\
\hline
CNN	 &63.8\scriptsize{$\pm$0.5}	&61.6\scriptsize{$\pm$0.5}	&51.1\scriptsize{$\pm$0.6}	&95.4\scriptsize{$\pm$0.3}	&49.8\scriptsize{$\pm$0.4}	&99.0\scriptsize{$\pm$0.2}  &70.1\\
\hline
DDC	 &64.4\scriptsize{$\pm$0.3}	&61.8\scriptsize{$\pm$0.4}	&52.1\scriptsize{$\pm$0.8}	&95.0\scriptsize{$\pm$0.5}	&\textbf{52.2}\scriptsize{$\pm$0.4}	&98.5\scriptsize{$\pm$0.4}  &70.6\\
\hline
DAN	 &65.8\scriptsize{$\pm$0.4}	&63.8\scriptsize{$\pm$0.4}	&\textbf{52.8}\scriptsize{$\pm$0.4}	&94.6\scriptsize{$\pm$0.5}	&51.9\scriptsize{$\pm$0.5}	&98.8\scriptsize{$\pm$0.6}  &71.3\\
\hline
D-CORAL	&\textbf{66.8}\scriptsize{$\pm$0.6}	&\textbf{66.4}\scriptsize{$\pm$0.4}	
&\textbf{52.8}\scriptsize{$\pm$0.2}	
&95.7\scriptsize{$\pm$0.3}	
&51.5\scriptsize{$\pm$0.3}	
&99.2\scriptsize{$\pm$0.1}  &\textbf{72.1}\\
\hline
\end{tabular}
}
\end{center}
\caption{Object recognition accuracies for all 6 domain shifts on the standard Office dataset with deep features, following the standard unsupervised adaptation protocol.}
\label{tab:result_office31_deep}
\vspace{-0.1in}
\end{table}

\paragraph{\textbf{Domain Adaptation Equilibrium}}

\begin{figure}[t]
\centering
\includegraphics[width=\linewidth]{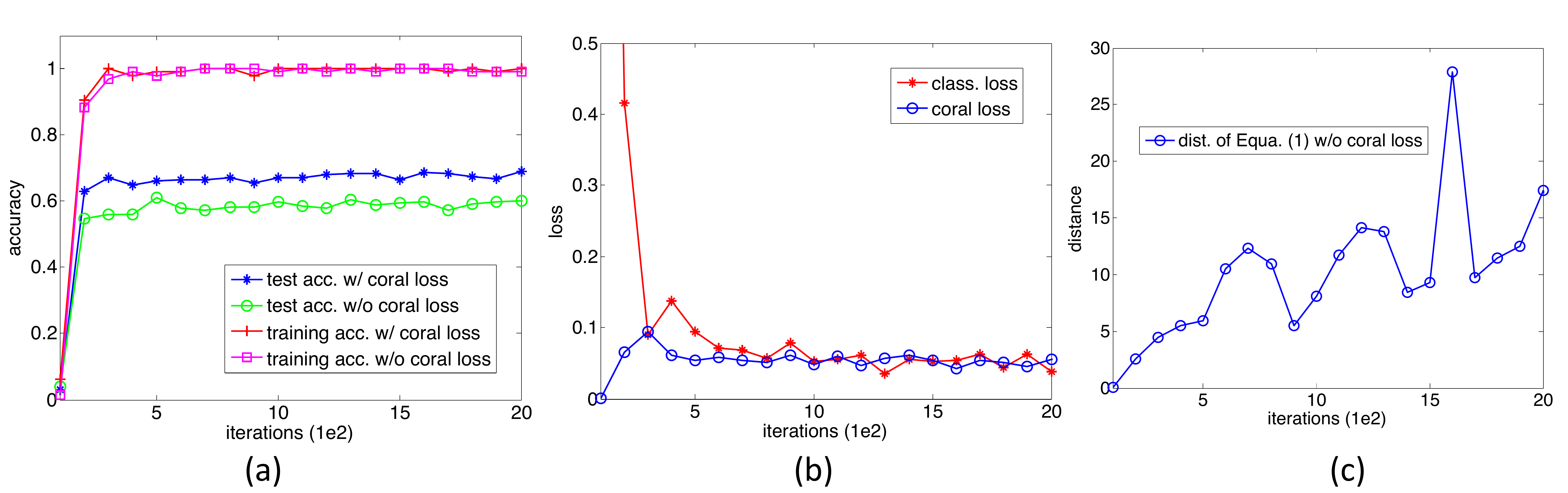}
\caption{Detailed analysis of shift A$\rightarrow$W for training w/ v.s. w/o CORAL loss. (a): training and test accuracies for training w/ v.s. w/o CORAL loss. We can see that adding CORAL loss helps achieve much better performance on the target domain while maintaining strong classification accuracy on the source domain. (b): classification loss and CORAL loss for training w/ CORAL loss. As the last fully connected layer is randomly initialized with $\mathcal{N}(0,0.005)$, CORAL loss is very small while classification loss is very large at the beginning. After training for a few hundred iterations, these two losses are about the same. (c): CORAL distance for training w/o CORAL loss  (setting the weight to 0). The distance is getting much larger ($\geqslant100$ times larger compared to training w/ CORAL loss).}
\label{fig:a_w}
\end{figure}

To get a better understanding of Deep CORAL, we generate three plots for domain shift A$\rightarrow$W. In Figure~\ref{fig:a_w}(a) we show the training (source) and testing (target) accuracies for training with v.s. without CORAL loss. We can clearly see that adding the CORAL loss helps achieve much better performance on the target domain while maintaining strong classification accuracy on the source domain. 

In Figure~\ref{fig:a_w}(b) we visualize both the classification loss and the CORAL loss for training w/ CORAL loss. As the last fully connected layer is randomly initialized with $\mathcal{N}(0,0.005)$,  in the beginning the CORAL loss is very small while the classification loss is very large. After training for a few hundred iterations, these two losses are about the same and reach an~\emph{equilibrium}. In Figure~\ref{fig:a_w}(c) we show the CORAL distance between the domains for training w/o CORAL loss (setting the weight to 0). We can see that the distance is getting much larger ($\geqslant100$ times larger compared to training w/ CORAL loss). Comparing Figure~\ref{fig:a_w}(b) and Figure~\ref{fig:a_w}(c), we can see that even though the CORAL loss is not always decreasing during training, if we set its weight to 0, the distance between source and target domains becomes much larger. This is reasonable as fine-tuning without domain adaptation is likely to overfit the features to the source domain. Our CORAL loss constrains the distance between source and target domain during the fine-tuning process and helps to maintain an~\emph{equilibrium} where the final features work well on the target domain.


\subsection{Object Detection}
\label{subsec:coral-lda}

\begin{figure}[b]
\centering
\includegraphics[width=0.9\linewidth]{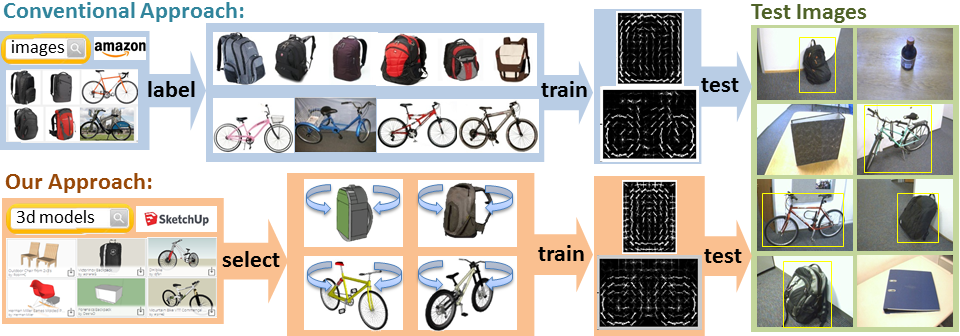}
\caption{\small Overview of our evaluation on object detection. With CORAL-Linear Discriminant Analysis (LDA), we show that a strong object detector can be trained from virtual data only.}
\label{chap4_fig:overview}
\end{figure}

Following protocol of~\cite{ICRA14}, we conduct object detection experiment on the Office dataset~\cite{saenko2010adapting} with HOG features. We use the same setting as~\cite{ICRA14}, performing detection on the \textit{Webcam} domain as the target (test) domain, and evaluating on the same 783 image test set of 20 categories (out of 31). As source (training) domains, we use: the two remaining real-image domains in Office,~\textit{Amazon} and~\textit{DSLR}, and two domains that contain virtual images only,~\textit{Virtual} and~\textit{Virtual-Gray}, generated from 3d CAD models. The inclusion of the two virtual domains is to reduce human effort in annotation and facilitate future research~\cite{sun16phdthesis}. Examples of~\textit{Virtual} and~\textit{Virtual-Gray} are shown in Figure~\ref{chap4_fig:virtual_data}. Please refer to~\cite{sun15virtualdataset} for detailed explanation of the data generation process.
We also compare to~\cite{ICRA14} who use corresponding ImageNet~\cite{imagenet} synsets as the source.
Thus, there are four potential source domains (two synthetic and three real) and one (real) target domain. The number of positive training images per category in each domain is shown in Table~\ref{chap4_tab:ntrain}. Figure~\ref{chap4_fig:overview} shows an overview of our evaluation.

\begin{figure}
\centering
\begin{tabular}{cc}
\includegraphics[width=5cm]{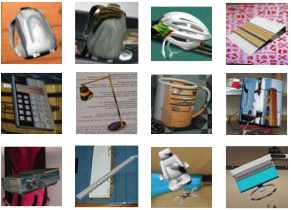} 
\hspace{\fill} &
\includegraphics[width=5cm]{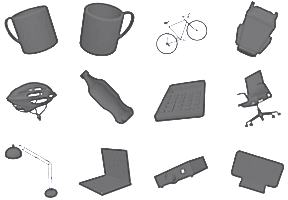} \\
(a)&(b)
\end{tabular}
\caption{\small Two sets of virtual images used in this section: (a) \textit{Virtual}: background and texturemap from a random real ImageNet image; (b) \textit{Virtual-Gray}:  uniform gray texturemap and white background. Please refer to~\cite{sun15virtualdataset} for detailed explanation of the data generation process.}
\label{chap4_fig:virtual_data}
\end{figure}

\begin{table}
\centering
\begin{small}
\begin{tabular}{|l|c|}
\hline
Domain & \# Training sample  \\
\hline \hline
Amazon & 20 \\ \hline
DSLR   & 8  \\ \hline
Virtual(-Gray) & 30 \\ \hline
ImageNet & 150-2000 \\
\hline
\end{tabular}
\caption{\small \# training examples for each source domain.}
\label{chap4_tab:ntrain}
\end{small}
\end{table}

\paragraph{\textbf{Effect of Mismatched Image Statistics}}
First, we explore the effect of mismatched precomputed image statistics on detection performance.
For each source domain, we train CORAL-LDA detectors using the positive mean from the source, and pair it with the covariance and negative mean of other domains. The virtual and the Office domains are used as sources, and the test domain is always Webcam. 
The statistics for each of the four domains were calculated using all of the training data,
following the same approach as~\cite{who}. The pre-computed statistics of 10,000 real images from PASCAL, as proposed in~\cite{who,ICRA14}, are also evaluated.

Detection performance, measured in Mean Average Precision (MAP), is shown in Table ~\ref{chap4_tab:performance}. We also calculate the normalized Euclidean distance between pairs of domains as $ (\|\mat{C^1} - \mat{C^2} \|) / (\|\mat{C^1}\|+\|\mat{C^2}\|) + (\|\vec{\mu_{0}^1}-\vec{\mu_{0}^2} \|) / (\| \vec{\mu_0^1}\| + \| \vec{\mu_0^2} \|) $, and show the average distance between source and target in parentheses in Table~\ref{chap4_tab:performance}.

From these results we can see a trend that larger domain difference leads to poorer performance. Note that larger difference to the target domain also leads to lower performance, confirming our hypothesis that both source and target statistics matter. Some of the variation could also stem from our assumption about the difference of means being the same not quite holding true. Finally, the PASCAL statistics from~\cite{who} perform the worst. Thus, in practice, statistics from either source domain or target domain or domains close to them could be used. However, unrelated statistics will not work even though they might be estimated from a very large amount of data as~\cite{who}. 

\begin{table}
\begin{small}
\begin{center}
\begin{tabular}{|l||c|c|c|c|c|c|}
\hline
  & Virtual & Virtual-Gray & Amazon  & DSLR & PASCAL\\ 
\hline
Virtual & 30.8 (\textit{0.1}) & 16.5 (\textit{1.0}) & 24.1 (\textit{0.6}) & 28.3 (\textit{0.2}) & 10.7 (\textit{0.5})\\ \hline
Virtual-Gray   & 32.3 (\textit{0.6}) & 32.3 (\textit{0.5}) & 27.3 (\textit{0.8}) & 32.7 (\textit{0.6}) & 17.9 (\textit{0.7})  \\ \hline
Amazon        & 39.9 (\textit{0.4}) & 30.0 (\textit{1.0}) & 39.2 (\textit{0.4}) & 37.9 (\textit{0.4}) & 18.6 (\textit{0.6})  \\ \hline
DSLR          &  68.2 (\textit{0.2}) & 62.1 (\textit{1.0}) & 68.1 (\textit{0.6}) & 66.5 (\textit{0.1}) & 37.7 (\textit{0.5})  \\ \hline
\end{tabular}
\end{center}
\caption{\small MAP of CORAL-LDA trained on positive examples from each row's source domain and background statistics from each column's domain. The average distance between each set of background statistics and the true source and target statistics is shown in parentheses. }\label{chap4_tab:performance}
\end{small}
\end{table}

\paragraph{\textbf{Unsupervised and Semi-supervised Adaptation}}
Next, we report the results of our unsupervised and semi-supervised adaptation technique. 
We use the same setting as~\cite{ICRA14}, in which three positive and nine negative labeled images per category were used for semi-supervised adaptation. 
Target covariance in Equation~\ref{eq:lda4} is estimated from 305 unlabeled training examples. We also followed the same approach to learn a linear combination between the unsupervised and supervised model via cross-validation. The results are presented in Table~\ref{tab:office}. Please note that our target-only MAP is 52.9 compared to 36.6 in~\cite{ICRA14}. This also confirms our conclusion that the statistics should come from a related domain. It is clear that both of our unsupervised and semi-supervised adaptation techniques outperform the method in~\cite{ICRA14}. Furthermore, \textit{Virtual-Gray} data outperforms~\textit{Virtual}, and \textit{DSLR} does best, as it is very close to the target domain (the main difference between $DLSR$ and $Webcam$ domains is in the camera used to capture images).

Finally, we compare our method trained on \textit{Virtual-Gray} to the results of adapting from ImageNet reported by~\cite{ICRA14}, in Figure ~\ref{chap4_fig:barplot}.
While their unsupervised models are learned from 150-2000 real ImageNet images per category and the background statistics are estimated from 10,000 PASCAL images, we only have 30 virtual images per category and the background statistics is learned from about 1,000 images. What's more, all the virtual images used are with uniform gray texturemap and white background. This clearly demonstrates the importance of domain-specific decorrelation, and shows that the there is no need to collect a large amount of real images to train a good classifier. 

\begin{table}
\begin{small}
\begin{center}
\begin{tabular}{|l||c|c||c|c|}
\hline
Source & Source-only~\cite{who} & Unsup-Ours & SemiSup~\cite{ICRA14} & SemiSup-Ours \\ 
\hline
Virtual & 10.7  &  27.9 & 30.7 & 45.2 \\ \hline
Virtual-Gray  & 17.9 & 33.0 & 35.0 & 54.7 \\ \hline
Amazon        & 18.6 & 38.9 & 35.8 & 53.0 \\ \hline
DSLR          & 37.7 & 67.1 & 42.9 & 71.4 \\ \hline
\end{tabular}
\includegraphics[width=\linewidth,height=1in]{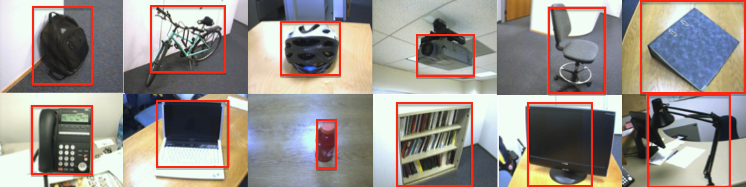}
\end{center}
\caption{\small Top: Comparison of the source-only~\cite{who} and semi-supervised adapted model of~\cite{ICRA14} with our unsupervised-adapted and semi-supervised adapted models. Target domain is $Webcam$. Mean AP across categories is reported on the $Webcam$ test data, using different source domains for training. Bottom: Sample detections of the DSLR-UnsupAdapt-Ours detectors.}
\label{tab:office}
\end{small}
\end{table}

\begin{figure}
\centering
\includegraphics[width=\linewidth]{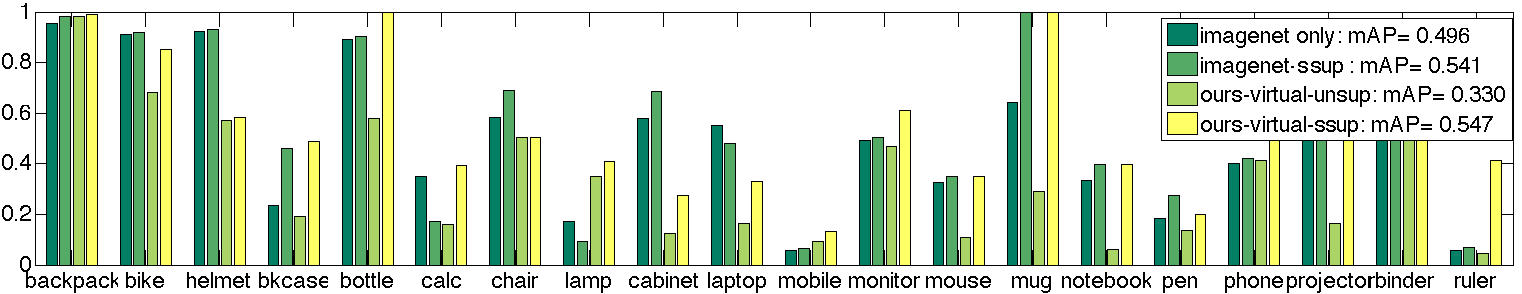}
\caption[\small Comparison of unsupervised and semi-supervised adaptation of virtual detectors using our method with the results of training on ImageNet and supervised adaptation from ImageNet reported in~\cite{ICRA14}]{\small Comparison of unsupervised and semi-supervised adaptation of virtual detectors using our method with the results of training on ImageNet and supervised adaptation from ImageNet reported in~\cite{ICRA14}. Our semi-supervised adapted detectors achieve comparable performance despite not using any real source training data, and using only 3 positive images for adaptation, and even outperform ImageNet significantly for several categories (\eg,~\textit{ruler}).}
\label{chap4_fig:barplot}
\end{figure}
\section{Conclusion}
\label{sec:con}
In this chapter, we described a simple, effective, and efficient method for unsupervised domain adaptation called CORrelation ALignment (CORAL). CORAL minimizes domain shift by aligning the second-order statistics of source and target distributions, without requiring any target labels. We also developed novel domain adaptation algorithms by applying the idea of CORAL to three different scenarios. In the first scenario, we applied a linear transformation that minimizes the CORAL objective to source features prior to training the classifier. In the case of linear classifiers, we equivalently applied the linear CORAL transform to the classifier weights, significantly improving efficiency and classification accuracy over standard LDA on several domain adaptation benchmarks. We further extended CORAL to learn a nonlinear transformation that aligns correlations of layer activations in deep neural networks. The resulting Deep CORAL approach works seamlessly with deep networks and can be integrated into any arbitrary network architecture to enable end-to-end unsupervised adaptation. One limitation of CORAL is that it captures second-order statistics only and may not preserve higher-order structure in the data. However, as demonstrated in this chapter, it works fairly well in practice, and can also potentially be combined with other domain-alignment loss functions. 

\bibliographystyle{plain}
\bibliography{coral.bib}

\end{document}